\documentclass[letterpaper, 10 pt, conference]{ieeeconf}  % Comment this line out
\IEEEoverridecommandlockouts                              % This command is only
\makeatletter\@mparswitchfalse\makeatother\normalmarginpar % to force the notes to the right
\usepackage[textwidth=1cm, textsize=scriptsize]{todonotes}

\usepackage[utf8]{inputenc}
\usepackage[T1]{fontenc}
\usepackage{xcolor}
\usepackage{cite}
\usepackage{amssymb}
\usepackage{graphicx}
\usepackage{xspace}
\usepackage{mathtools}
\usepackage{flushend}
\usepackage{subcaption}
\usepackage{algorithm}
\usepackage{algpseudocode}
\algrenewcommand\algorithmicrequire{\textbf{Input:}}
\newcommand{\CommentState}[1]{\Statex\hspace{\algorithmicindent}{\color{blue}// #1}}

\usepackage{amsmath,amsfonts,amssymb}
\usepackage{bm,fixmath} % for math bold letters
\newtheorem{assumption}{Assumption}
\newtheorem{lemma}{Lemma}

\newtheorem{remark}{Remark}
\newtheorem{theorem}{Theorem}

\newtheorem{definition}{Definition}

\allowdisplaybreaks
\newcommand{\alg}{LT-ADMM-DP\xspace}
\newcommand{\tgrad}{t_g\xspace}
\newcommand{\tcomm}{t_c\xspace}
\DeclareMathOperator*{\argmin}{arg\,min}

\newcommand{\N}{\mathbb{N}}
\newcommand{\R}{\mathbb{R}}

\title{\LARGE \bf
Communication-Efficient Distributed Learning with Differential Privacy
}
\author{%
    Xiaoxing~Ren$^{1\star}$, Yuwen~Ma$^{2}$,	 Nicola~Bastianello$^{3}$, Karl~H.~Johansson$^{3}$, \\Thomas~Parisini$^{4,5,6}$, Andreas A. Malikopoulos$^{1,7}$
\thanks{The work of X.R. and A.M. was supported in part by NSF under Grants CNS-2401007, CMMI-2348381, IIS-2415478, and in part by MathWorks. The work of N.B. and K.H.J. was supported by Swedish Research Council Distinguished Professor Grant 2017-01078 Knut and Alice Wallenberg Foundation Wallenberg Scholar Grant.  }
         \thanks{$^{1}$School of Civil and Environmental Engineering, Cornell University, Ithaca, New York,  United States. {\tt\footnotesize xr49@cornell.edu}}
    \thanks{$^{2}$Department of Electronic and Electrical Engineering, University College London, London WC1E 7JE, United Kingdom.  {\tt \footnotesize yuwen.ma.24@ucl.ac.uk}}
	\thanks{$^{3}$School of Electrical Engineering and Computer Science, and Digital Futures, KTH Royal Institute of Technology, Stockholm, Sweden. {\tt \footnotesize nicolba@kth.se, kallej@kth.se}}
 \thanks{$^{4}$Department of Electrical and Electronic Engineering, Imperial College London, London, United Kingdom, $^{5}$Department of Electronic Systems, Aalborg University, Denmark, $^{6}$Department of Engineering and Architecture, University of Trieste, Trieste, Italy. {\tt \footnotesize t.parisini@imperial.ac.uk}}%
    \thanks{$^{7}$Applied Mathematics, Systems Engineering, Mechanical Engineering, Electrical \& Computer Engineering, Cornell University, Ithaca, New York,  United States. {\tt \footnotesize amaliko@cornell.edu}}
	\thanks{$^{\star}$Corresponding author.}%
}

\begin{document}

\maketitle

% don't print page numbers
\thispagestyle{empty}
\pagestyle{empty}
% print page numbers
% \thispagestyle{plain}
% \pagestyle{plain}

%%%%%%%%%%%%%%%%%%%%%%%%%%%%%%%%%%%%%%%%%%%%%%%%%%%%%%%%%%%%%%%%%%%%%%%%%%%%%%%%

%%%%%%%%%%%%%%%%%%%%%%%%%%%%%%%%%%%%%%%%%%%%%%%%%%%%%%%%%%%%%%%%%%%%%%%%%%%%%%%%

\begin{abstract}
We address nonconvex learning problems over undirected networks. In particular, we focus on the challenge of designing an algorithm that is both communication-efficient and that guarantees the privacy of the agents' data. The first goal is achieved through a local training approach, which reduces communication frequency. The second goal is achieved by perturbing gradients during local training, specifically through gradient clipping and additive noise. We prove that the resulting algorithm converges to a stationary point of the problem within a bounded distance. Additionally, we provide theoretical privacy guarantees within a differential privacy framework that ensure agents' training data cannot be inferred from the trained model shared over the network. We show the algorithm's superior performance on a classification task under the same privacy budget, compared with state-of-the-art methods.
\end{abstract}
%We consider distributed nonconvex learning problems with privacy guarantees over undirected networks.
% To alleviate the heavy communication burden in real-world training problems, we propose a fully distributed algorithm without a central coordinator, which leverages multi-step local training to reduce the frequency of communication. Furthermore, to guarantee differential privacy at the sample level, we perturb the locally aggregated gradients with Gaussian noise. This explicitly conceals the participation of any individual record within the local training phase, thereby limiting sensitive information leakage through the learned model. Experiments on classification problems demonstrate that the proposed algorithm reduces communication overhead while maintaining high classification accuracy and ensuring differential privacy guarantees.

\section{Introduction}
Distributed learning over multi-agent networks has received considerable attention due to its ability to solve large-scale optimization and inference problems using decentralized data and computation resources. In this context, a collection of agents cooperatively learns a global model by iteratively performing local computations and exchanging information with neighboring agents, without relying on a centralized coordinator. This architecture is particularly appealing in applications where data are inherently distributed \cite{Malikopoulos2022a,Malikopoulos2024} and privacy-sensitive, such as smart grids, connected vehicles, and robotics~\cite {molzahn_survey_2017,serrani:2025,nedic_distributed_2018,le2025combining,le2024multi}.

A core objective in algorithm design for distributed learning is achieving efficiency: computation-efficiency to accommodate limited local computing resources, and communication-efficiency to cope with limited communication bandwidth.
A widely used approach to address the first goal, computation-efficiency, is to employ stochastic gradients, which allow agents to process only a subset of the local data~\cite{basu2020qsparse}.
Several approaches have been proposed to achieve the second goal, communication-efficiency. The main approaches are (i) reducing the frequency of communications, allowing agents to perform several local training epochs before communicating their results, (ii) compressing the communications~\cite{alghunaim_local_2023,ren2024distributed,ren2025communication_learning,zhang2023innovation,huang2024cedas,ren2025jointly}.

Designing efficient distributed algorithms is an important objective, but learning over networks raises critical privacy concerns for the participating agents.
Indeed, it has been demonstrated that models trained by the agents and shared over the network can be used to infer knowledge of the agents' private data \cite{abadi2016deep}, even reconstructing raw training data \cite{zhu2019deep}.
Differential privacy (DP) has thus been proposed to protect agents' data from leaks, and provide strong mathematical guarantees that privacy is indeed preserved \cite{dwork2006calibrating,chaudhuri2011differentially}.
The intuition is that a learning algorithm is differentially private if the presence or absence of any one data point from the training set does not noticeably change the trained model. In other words, the trained model does not expose enough information to reconstruct whether any single datapoint was used during training.
Formally, DP can be achieved by perturbing the model during training, usually by adding Gaussian or Laplacian noise \cite{chaudhuri2011differentially}. More recently, gradient clipping has been integrated alongside noise as a privacy mechanism, both with fixed \cite{abadi2016deep} and dynamic \cite{noble2022differentially} clipping thresholds.
%
%Since it was first proposed for single agent learning.
DP has also been applied to distributed learning \cite{liu_survey_2024}, to preserve all participating agents' privacy against dishonest peers or external eavesdroppers. 
Distributed learning that guarantees local privacy by adding noise to the gradients has been widely studied \cite{li2025convergence,huang2025differential,ding2021differentially}.

Despite advances in differentially private distributed learning, the challenge of designing algorithms that are both efficient and private remains open. Therefore, in this paper, we offer the following contributions:
\begin{itemize}
    \item We propose a distributed algorithm, \alg (Local Training ADMM with Differential Privacy), to solve nonconvex learning problems. The algorithm employs local training to reduce communication frequency. For computation-efficiency, during local training the agents use stochastic gradients. The gradients are both clipped and perturbed by noise, for privacy preservation.

    \item We analyze the convergence of \alg to a bounded distance from a stationary point of the nonconvex problem, characterizing how this distance depends on features of the problem and hyperparameters.

    \item We prove that \alg is differentially private owing to the use of clipped and noisy stochastic gradients during local training.

    \item We provide a numerical comparison of \alg with state-of-the-art algorithms for a classification problem with nonconvex regularized, which shows its superior performance and privacy guarantees.
\end{itemize}

The remainder of the paper is organized as follows. 
In Section~II, we provide the problem formulation and preliminary results.
In Section~III, we present the algorithm and, in Section~IV, the convergence analysis. In Section ~V, we discuss the privacy analysis and, in Section~VI, we provide numerical studies to validate the effectiveness of the proposed framework. 
Finally, in Section~VII, we draw concluding remarks and discuss potential directions for future research.

\section{Problem Formulation and Preliminaries}

\subsection{Notation} We denote the gradient of a differentiable function $f$ by $\nabla f$. $\mathbf{1}_n \in\mathbb{R}^n$, with $n\in \mathbb{N}$, denotes the vector with all elements equal to $1$, $\mathbf{I} \in \R^{n \times n}$ denotes the identity matrix and $\mathbf{0} \in \R^{n \times n}$ the zero matrix. We use $\otimes$ to denote the Kronecker product and
$\langle x, y \rangle =  \sum_{h=1}^n x_h y_h$ to represent the inner product of two vectors $x, y \in \R^n$. We use
$\| \cdot \|$ to denote the Euclidean norm and induced matrix norm.

\subsection{Problem formulation}
Consider a network $\mathcal{G} = (\mathcal{V},\mathcal{E})$ with $N$ agents, with each agent storing a local dataset which defines the local cost as
\begin{equation}\label{eq:erm-cost}
    f_{i}({x}) = \frac{1}{m_i} \sum_{h=1}^{m_i} f_{i,h}(x) \, ,
\end{equation}
where $f_{i,h}: \mathbb{R}^{n} \rightarrow \mathbb{R}$ is the loss function associated to data point $h \in \{ 1, \ldots, m_i \}$, $x$ is the model parameter to be determined.
The objective is for the agents to pool their resources together and cooperatively solve the following problem,  the objective is the sum of local costs~\eqref{eq:erm-cost} and the constraint imposes consensus on a shared trained model
\begin{equation}\label{eq:optimization-problem}
    \min_{x_i \in\mathbb{R}^{n}, \ i \in \mathcal{V}} \ \ \frac{1}{N} \sum_{i=1}^{N}f_{i}({x}_i) \quad \text{s.t.} \ \ x_1 = x_2 = \cdots = x_N \, ,
\end{equation}
where $x_i$ is the local model parameters.
We denote the optimal solution of~\eqref{eq:optimization-problem} by $\mathbf{X}^* = \mathbf{1}_N \otimes x^*$, and $x^* = \argmin_{x \in \R^n} F(x)$, where $ F(x) = \sum_{i = 1}^N f_i(x)$.

We introduce the following standard assumptions on the cost functions and network.

\begin{assumption}\label{as:local-costs}
The loss function of each agent $i \in \mathcal{V}$ is globally  $L$-smooth with $L > 0$: 
$\forall x, y \in \mathbb{R}^n$,
$\Vert \nabla f_{i}(x)-\nabla f_{i}(y)\Vert  \leq L \Vert x-y \Vert$.
\end{assumption}
\begin{assumption}\label{as:graph}
$\mathit{\mathcal{G}} = (\mathcal{V},\mathcal{E})$ is connected and undirected.
\end{assumption}
\begin{assumption}\label{ass:Bounded gradient variation}
For any $x \in \mathbb{R}^n$, $\forall i$, there exists a constant $\sigma_f>0$ such that $\forall i$, $\big\| \nabla F({x})-\nabla f_i(x)\big\|
\le \sigma_f$.
\end{assumption}
Assumption~\ref{ass:Bounded gradient variation} characterizes the difference between each local gradient and the global gradient, and hence reflects the level of data heterogeneity across agents. Similar bounded-gradient-variation assumption has also been adopted in \cite{huang2025differential}.

\subsection{Differential privacy}
For the reader's convenience,  we review some fundamental notions and results in differential privacy (DP), see \cite{dwork2014algorithmic,mironov2017renyi} for additional details.
The definition of DP, introduced below, is based on the concept of \textit{adjacent datasets}: two datasets are adjacent if they differ by exactly one datapoint, which is present in only one of the two datasets.

\begin{definition}\label{de1}
    A randomized mechanism $\mathcal{M}: \mathcal{D} \rightarrow \mathcal{R}$ (in this context: learning algorithm) satisfies $(\epsilon, \delta)$-DP, $\epsilon, \delta > 0$, if, for any pair of adjacent datasets $D, D' \in \mathcal{D}$ and any subset of outputs $S \subseteq \mathcal{R}$, the following inequality holds:
    \begin{align}\label{dpe1}
    \text{Pr}[\mathcal{M}(D) \in S] \le e^{\epsilon} \text{Pr}[\mathcal{M}(D') \in S] + \delta.
    \end{align}
\end{definition}

The privacy budget $\epsilon$ quantifies the statistical indistinguishability of the mechanism's outputs. A smaller $\epsilon$ signifies a higher level of privacy protection. The parameter $\delta$ represents the failure probability, allowing for a small margin when the $\epsilon$-DP bound may not hold.% (as perfect DP is hard to guarantee in practice).
Intuitively, since the adjacent datasets only differ in one datapoint, their being statistically indistinguishable implies that that datapoint cannot be reconstructed from the mechanism's output.

When applying $(\epsilon, \delta)$-DP to iterative learning algorithms, the drawback is that the resulting privacy bounds tend to be loose. Therefore, in this paper, we employ Rényi Differential Privacy (RDP) \cite{mironov2017renyi}, which allows us to derive tighter privacy bounds. 
Recalling that the Rényi divergence of order $\alpha \in (1, \infty)$ between distributions $P$ and $Q$ is defined as:
\begin{align*}
    D_\alpha(P \parallel Q) = \frac{1}{\alpha - 1} \log {\displaystyle \mathbb{E}_{x \sim Q} }\left[ \left( \frac{P(x)}{Q(x)} \right)^\alpha \right],
\end{align*}
RDP is characterized as follows.

\begin{definition}\label{de3}
    A randomised mechanism $\mathcal{M}: \mathcal{D} \rightarrow \mathcal{R}$ satisfies $(\alpha, \rho)$-RDP, $\alpha, \rho >0$, if, for any pairs of adjacent datasets $D, D' \in \mathcal{D}$, the following condition holds:
    \begin{align*}
        D_\alpha(\mathcal{M}(D) \parallel \mathcal{M}(D')) \le \rho.
    \end{align*}
\end{definition}

Importantly, RDP and DP are related by the following result.

\begin{lemma}[\cite{mironov2017renyi}, Proposition 3]\label{lem2}
    If a randomised mechanism $\mathcal{M}$ satisfies $(\alpha, \rho(\alpha))$-RDP, it satisfies $(\epsilon, \delta)$-DP for any $\delta \in (0, 1)$, with the privacy budget
\begin{align}\label{dpe2}
    \epsilon = \rho(\alpha) + \frac{\log(1/\delta)}{\alpha-1}.
\end{align}
\end{lemma}

As discussed in the Introduction, in section~\ref{sec:algorithm-design} we propose a distributed learning algorithm that employs additive Gaussian noise as a privacy-preserving mechanism. This mechanism indeed guarantees Rényi differential privacy, as recalled in Lemma~\ref{lem1} below, which is based on the following definition of $\ell_2$-sensitivity.

\begin{definition}\label{de2}
     Let $f: \mathcal{D} \rightarrow \mathbb{R}^n$ be a deterministic function (in this context: \eqref{eq:erm-cost}). The $\ell_2$-sensitivity of $f$, denoted by $\Delta_{2,f}$, is the maximum Euclidean distance between the outputs for any pair of adjacent datasets $D, D' \in \mathcal{D}$:
     \begin{align*}
         \Delta_{2,f} = \max_{D, D' \in \mathcal{D}} \|f(D) - f(D')\|_2.
     \end{align*}
\end{definition}

\begin{lemma}[\cite{mironov2017renyi}, Proposition 7]\label{lem1}
    Consider a query function $f$ with $\ell_2$-sensitivity $\Delta_{2,f}$. The Gaussian mechanism, which perturbs the output of $f$ by adding $\mathcal{N}(0, \sigma^2 \mathbf{I})$ noise, satisfies $(\alpha, \rho(\alpha))$-RDP for any $\alpha > 1$ with $$\rho(\alpha) = \frac{\alpha \Delta_{2,f}^2}{2 \sigma^2}.$$
\end{lemma}

% \subsection{Problem Statement}

% Building upon the system model \eqref{eq:optimization-problem} and privacy framework, we formally state the core objective as follows.

% \begin{problem}\label{pro1}
% Design a \textit{differentially private and communication efficient distributed learning protocol} for the multi-agent network. The primary goal is to collaboratively minimise the objective function \eqref{eq:optimization-problem} while ensuring that the participation of any individual data record in the local training process is protected under a rigorous $(\epsilon, \delta)$-DP guarantee.
% \end{problem}

\section{Algorithm Design}\label{sec:algorithm-design}
This section proposes a distributed learning algorithm that achieves both efficiency and privacy.
Formally, we want to design an algorithm that allows the agents in $\mathcal{G}$ to cooperatively solve~\eqref{eq:optimization-problem} while ensuring $(\epsilon, \delta)$-DP for each agent's dataset.

Our algorithm is based on LT-ADMM~\cite{ren2024distributed,ren2025communication_learning}, into which we integrate a privacy preserving mechanism.
In particular, LT-ADMM is characterized by the local updates
\begin{subequations}\label{eq:admm}
\begin{align}
 &\phi_{i,k}^0 =x_{i, k} \, , \nonumber
 \\&
\phi_{i,k}^{t+1}  =  \phi_i^{t}-( \gamma g_i(\phi_{i,k}^{t}) + \beta (\rho\left|\mathcal{N}_i\right| x_{i, k} - \sum_{j \in \mathcal{N}_i} z_{i j, k} ) ), \nonumber 
\\& \qquad t =0, \ldots, \tau-1 \, , \nonumber \\&
x_{i, k+1}  =\phi_{i,k}^\tau  \, , \label{eq:local-training} \\
&z_{i j, k+1} = \frac{1}{2} z_{i j, k}- \frac{1}{2} \left(z_{j i, k}-2 \rho x_{j, k+1}\right), \label{eq:admm-z}
\end{align}
\end{subequations}
where $\rho>0$ is a penalty parameter, $\gamma, \beta>0$ are the local stepsizes, $g_i(\phi)$ is a local gradient estimator, and 
$z_{ij,k}$ and $z_{ji,k}$ are bridge variables for edge $(i,j)$.
During update~\eqref{eq:local-training}, agents perform $\tau \in \N$ steps of local gradient descent to approximate the solution of a local minimization problem~\cite{bastianello_asynchronous_2021}.%, which in general lacks a closed form, especially in learning problems.
Update~\eqref{eq:admm-z} then propagates the result of the local training through the network.

Now, to fully characterize the proposed algorithm, we need to define the gradient estimator $g_i(\phi)$. When computation efficiency is the only concern, a stochastic gradient estimator is enough \cite{ren2024distributed}:
\begin{equation} \label{eq:stochastic_gradient}
    {g}_{i,k}^t =  \frac{1}{|\mathcal{B}_i|} \sum_{h \in \mathcal{B}_i} \nabla f_{i, h}\left( \phi_{i, k}^t \right),
\end{equation}
where $\mathcal{B}_i$ is {a mini-batch of size $|\mathcal{B}_i| < m_i$, generated via uniform sampling without replacement from the local index set $\{ 1, \ldots, m_i\}$}.
However, in this paper, we want to concurrently achieve efficiency and privacy. Hence, we employ the following \textit{perturbed gradient estimator}:
\begin{equation}\label{eq:clip-gradient}
    g_i\left(\phi_{i,k}^t\right)  = \mu_{i,k}^t  {g}_{i,k}^t + e_{i,k}^t,
\end{equation}
which clips the stochastic gradient~\eqref{eq:stochastic_gradient} by
$\mu_{i,k}^t = \frac{\zeta}{\zeta + \lVert {g}_{i,k}^t \rVert_2}$, with threshold $\zeta>0$, and moreover adds Gaussian noise to it, $e_{i,k}^t\sim \mathcal{N}\!\left(0_n,\;\sigma_e^2 I_n\right)$, $\sigma_e>0$.

The resulting \alg (Local Training ADMM with Differential Privacy) is described in Algorithm~\ref{alg:lt-admm-dp}.
\begin{algorithm}[!ht]
\caption{LT-ADMM-DP}
\label{alg:lt-admm-dp}
\begin{algorithmic}[1]
\Require For each node $i$, initialize $x_{i,0}= z_{ij, 0}$, $j \in \mathcal{N}_i$. Set the penalty parameter $\rho>0$, the number of local training steps $\tau >0$, the normalization constant $\zeta>0$, and the step sizes $\gamma, \beta>0$.
 
	\For{$k = 0,1,\ldots$ every agent $i$}
 \CommentState{local training}

    \State $\phi_{i,k}^0 = x_{i,k}$, 
    % {$r_{i, h, k}^{0}  = x_{i, k}$, for all $h \in \{ 1, \ldots. m_i \}$}
    % $\mathbf{u}_{i, k+1}  =  (1- \eta) \mathbf{u}_{i, k} + \eta  \hat{\mathbf{x}}_{i,k}$, $s_{ij,k+1} = \hat{\mathbf{z}}_{ij,k}$
    
    \For{$t = 0, 1, \ldots, \tau-1$}

\State Draw the batch $\mathcal{B}_i$ uniformly at random

\State Update the stochastic gradient according to \eqref{eq:stochastic_gradient}

\State Clip and add noise to the stochastic gradient according to \eqref{eq:clip-gradient}

\State Update $\phi_{i,k}$ according to \eqref{eq:local-training}

    \EndFor
    
\State Set $x_{i,k+1} = \phi_{i,k}^\tau$
    
	\CommentState{communication}
    \State Transmit $z_{j i, k}-2 \rho x_{j, k+1}$ to each neighbor $j \in \mathcal{N}_i$, and receive the corresponding transmissions
    
\CommentState{auxiliary update}
\State Update $ {{z}}_{ij,k+1}$ according to \eqref{eq:admm-z}
	\EndFor
\end{algorithmic}
\end{algorithm}
Note that the agents employ local training to reduce the frequency of communications, during which they use clipped and perturbed stochastic gradients.

\section{Convergence Analysis}
We start by introducing the following standard assumption for the local stochastic gradient estimators~\eqref{eq:stochastic_gradient}.

\begin{assumption} \label{as:sgd}
For all $i \in \mathcal{V}$ and $\phi \in \R^n$, the stochastic gradient $ \frac{1}{|\mathcal{B}_i|} \sum_{h \in \mathcal{B}_i} \nabla f_{i, h}\left( \phi \right)$, in \eqref{eq:stochastic_gradient} is unbiased with variance bounded by $\sigma_g > 0$, that is: $ \mathbb{E}\left[   \frac{1}{|\mathcal{B}_i|} \sum_{h \in \mathcal{B}_i} \nabla f_{i, h}\left( \phi \right) -\nabla f_i(\phi)  \right]=0 $ and $\mathbb{E}\left[\left\|\frac{1}{|\mathcal{B}_i|} \sum_{h \in \mathcal{B}_i} \nabla f_{i, h}\left( \phi \right) - \nabla f_i(\phi)\right\| \right] \leq \sigma_g $.
\end{assumption}

Next, we present the main convergence result of this paper. The proof is given in the Appendix.

\begin{theorem} \label{theo:nonconvex_sgd_converge}
    Let Assumptions~\ref{as:local-costs}--\ref{as:sgd} hold. Let $\zeta > 8 \sigma_g $, $\beta <  \frac{2}{\tau\lambda_u\rho}$,  $\gamma \leq \mathcal{O}(\frac{\lambda_l}{L \tau^2 })$ (see \eqref{eq:gamma1} for the detailed bound).
    Then the output of Algorithm~\ref{alg:lt-admm-dp} satisfies
    \begin{align} \label{the:main_convergence}
        &  \frac{1}{K}\sum_{\substack{0 \leq k \leq K-1 \\ k \in \mathbf{K_1}}} (\frac{\zeta}{4} - 2 \sigma_g )\|\nabla F(\bar x_k)\| \nonumber
        \\&+  \frac{1}{K} \sum_{\substack{0 \leq k \leq K-1 \\ k \in \mathbf{K_2}}} \frac{1}{4}\|\nabla F(\bar x_k)\|^2 
    \nonumber
    \\& \leq \mathcal{O}\left(\frac{ F(\bar{x}_{0})-F(x^*)}{K \gamma \tau}\right) + \mathcal{O}\left( \gamma \tau ( \zeta^2 + \sigma^2_e)\right)  \nonumber
    \\&  
    +  \mathcal{O}\left(  \sigma_f^2 + \zeta \sigma_g \right) + \mathcal{O}\left( \frac{\|\widehat{\mathbf{d}}_{0}\|^2}{\rho^2 K N} \right),
\end{align}
where $k \in \mathbf{K}_1$ if $ \| \nabla F(\bar x_k) \| \geq \zeta$, $k \in \mathbf{K}_2$ otherwise;  $\lambda_u$ and $\lambda_l$ denote the largest and the smallest nonzero eigenvalues of the Laplacian matrix $\mathcal{L}$ of $\mathcal{G}$, respectively; and $\|\widehat{\mathbf{d}}_0\|$ is a constant related to the initial conditions (see Appendix for the detailed definition).
\end{theorem}

\begin{remark}
Theorem~\ref{theo:nonconvex_sgd_converge} shows that the upper bound of the stepsize $\gamma$ is proportional to the network connectivity (the algebraic connectivity $\lambda_l$). Thus, ``less connected" graphs (smaller $\lambda_l$) result in smaller stepsizes.
The right-hand side in \eqref{the:main_convergence} reveals a trade-off: a larger $\gamma$ (or $\tau$) improves convergence through the term $\mathcal{O} (\tfrac{1}{K\gamma\tau})$, but it also enlarges the steady-state error.%, requiring careful tuning. % Thus, $\gamma$ should be tuned to balance the convergence speed and steady error, and a similar trade-off also applies to $\tau$.
Additionally, the steady-state error depends on several quantities: the gradient variation, $\sigma_f$, the stochastic gradient variance $\sigma_g$, and the noise variance $\sigma_e^2$. While the former two cannot be controlled, the latter can. Thus, in principle, reducing the perturbation during local training improves optimality; but, as we will show in the next section, this would come at the cost of less privacy.
\end{remark}

\section{Privacy Analysis}\label{sec:privacy}
This section establishes privacy guarantees for Algorithm~\ref{alg:lt-admm-dp}.% within the differential privacy framework.
In particular, we prove that a malicious agent observing the results of agent $i$'s local training, $x_{i,k}$, cannot distinguish which of the adjacent minibatches $\mathcal{B}_i$ and $\mathcal{B}'_i$ was used by the agent.

\begin{theorem}\label{thm:privacy}
Consider agent $i \in \mathcal{V}$ with local cost \eqref{eq:erm-cost}. Given any $\delta_i \in (0, 1)$, Algorithm~\ref{alg:lt-admm-dp} guarantees $(\epsilon_i, \delta_i)$-DP for $i$'s local minibatches $\mathcal{B}_i$ across all $K$ global iterations, with $\epsilon_i$ given by
\begin{equation}\label{theorem2_eq}
    \epsilon_i = \frac{2 K \tau \zeta^2 |\mathcal{B}_i|^2}{\sigma_e^2 m_i^2} + \frac{2\zeta|\mathcal{B}_i|}{\sigma_em_i} \sqrt{2 K \tau \log(1/\delta_i)}.
\end{equation}
\end{theorem}
\begin{proof}
    The proof is divided into two steps: (i) we prove RDP of the algorithm and (ii) translate it into DP.

    \textit{Step (i)}: Let $g_i(\phi_{i,k}^t)$ in \eqref{eq:clip-gradient} denote the differentially private sub-mechanism $\mathcal{A}_{i,k}$. Correspondingly, $f_i \triangleq \mu_{i,k}^t{g}_{i,k}^t$ represents the non-privatised query function with respect to the sensitive dataset $\mathcal{B}_i$. By Definition \ref{de2}, the $\ell_2$-sensitivity of $f_i$ is:
    \begin{align*}
        \Delta_{2,f_i}&= \max_{\mathcal{B}_i, \mathcal{B}'_i} \|f_i(\mathcal{B}_i) - f_i(\mathcal{B}'_i)\|
        \\
        &\leq \max_{\mathcal{B}_i}\|f_i(\mathcal{B}_i)\|+ \max_{\mathcal{B}'_i}\|f_i(\mathcal{B}'_i)\|\leq2\zeta.
    \end{align*} By applying Lemma 1 and the subsampling amplification bounds from \cite[Appendix C.7]{wang2019subsampled} to \eqref{eq:clip-gradient}, it follows that for any order $1 < \alpha_i \ll \frac{\sigma_e^2 m_i}{|\mathcal{B}_i|}$, the differentially private sub-mechanism $\mathcal{A}_{i,k}$ satisfies $(\alpha_i, \frac{2\alpha_i\zeta^2 |\mathcal{B}_i|^2}{\sigma_e^2m_i^2 })$-RDP.

    To analyse the privacy guarantees of Algorithm \ref{alg:lt-admm-dp}, we define the mechanism $\mathcal{M}_{i,k}$ as Steps 2--10 executed by agent $i$ at iteration $k$. Observe from \eqref{eq:local-training} that $\mathcal{M}_{i,k}$ comprises $\tau$ sequential executions of the sub-mechanism $\mathcal{A}_{i,k}$. As established, each application of $\mathcal{A}_{i,k}$ satisfies $\left(\alpha_i, \frac{2\alpha_i\zeta^2 |\mathcal{B}_i|^2}{\sigma_e^2 m_i^2}\right)$-RDP. Invoking the RDP sequential composition theorem \cite{mironov2017renyi} alongside the post-processing property \cite{dwork2014algorithmic}, the cumulative privacy cost of $\mathcal{M}_{i,k}$ over $\tau$ iterations is the exact sum of the individual step costs. Consequently, $\mathcal{M}_{i,k}$ guarantees $(\alpha_i, \rho_{\mathcal{M}_{i,k}}(\alpha_i))$-RDP, where:
    \begin{equation}\label{dpe3}
        \rho_{\mathcal{M}_{i,k}}(\alpha_i) = \sum_{t=0}^{\tau-1} \frac{2\alpha_i\zeta^2|\mathcal{B}_i|^2}{\sigma_e^2m_i^2} = \frac{2 \tau \alpha_i \zeta^2 |\mathcal{B}_i|^2}{\sigma_e^2 m_i^2}.
    \end{equation} 
    To characterise the cumulative privacy loss throughout the entire training process, we employ the RDP sequential composition theorem \cite{mironov2017renyi} over all $K$ global iterations. This ensures the overall mechanism $\mathcal{M}_i$ satisfies $\left(\alpha_i, \frac{2 K \tau \alpha_i \zeta^2 |\mathcal{B}_i|^2}{\sigma_e^2 m_i^2}\right)$-RDP.
    
    \textit{Step (ii)}: Next, we translate this result into a DP bound. Applying Lemma~\ref{lem2} we have that Algorithm~\ref{alg:lt-admm-dp} satisfies $(\epsilon_i, \delta_i)$-DP for any $\delta_i > 0$ with
    \begin{align}\label{dpe4}
          \epsilon_i = \min_{\alpha_i > 1} \left( \frac{2 K \tau \alpha_i \zeta^2 |\mathcal{B}_i|^2}{\sigma_e^2 m_i^2} + \frac{\log(1/\delta_i)}{\alpha_i-1} \right),
    \end{align}
    where we minimize over the arbitrary $\alpha_i$ (assuming that $\alpha_i \ll \frac{\sigma_e^2 m_i}{|\mathcal{B}_i|}$). Solving~\eqref{dpe4} thus yields the optimal value of $\alpha_i$ as $\alpha^*_i = 1 + \sqrt{\frac{\sigma_e^2 m_i^2 \log(1/\delta_i)}{2 K \tau \zeta^2 |\mathcal{B}_i|^2}}$, which in turn yields the privacy budget
    \begin{align*}
        \epsilon_i = \frac{2 K \tau \zeta^2 |\mathcal{B}_i|^2}{\sigma_e^2 m_i^2} + \frac{2\zeta|\mathcal{B}_i|}{\sigma_em_i} \sqrt{2 K \tau \log(1/\delta_i)}.
    \end{align*}
    Finally, we need to ensure that indeed the assumption of $\alpha_i \ll \frac{\sigma_e^2 m_i}{|\mathcal{B}_i|}$ is satisfied. But this is clearly true provided that the hyperparameters $K$, $\tau$ are chosen sufficiently large.
\end{proof}

\section{Numerical Results}\label{sec:numerical}
In this section, we present numerical results that compare \alg with state-of-the-art distributed methods PORTER~\cite{li2025convergence} and PriSMA~\cite{huang2025differential}. 
% These algorithms rely on a central coordinator, as opposed to the fully decentralized \alg.
%
We focus on a {classification task} defined by the local costs
$$
    f_i(x) = \frac{1}{m_i} \sum_{h = 1}^{m_i} \log\left( 1 + \exp\left( - b_{i,h} a_{i,h}^\top x \right) \right) + \epsilon \sum_{\ell = 1}^n \frac{[x]_\ell^2}{1 + [x]_\ell^2},
$$
where $[x]_\ell$ is the $\ell$-th component of $x \in \R^n$, and $a_{i,h} \in \R^n$ and $b_{i,h} \in \{ -1, 1 \}$ are the pairs of feature vector and label.
We choose a ring network with $N = 10$, $n = 5$, $m_i = 1000$, and $|\mathcal{B}_i| = 8$.
% the initial conditions are randomly drawn from $ \mathcal{N}(0, 20 \mathrm{I}_n)$.
% For the simulations, we compare the proposed method with two privacy-preserving distributed learning algorithms.

We select the following hyperparameters: for \alg, $\gamma = \beta = 0.1$, $\rho = 0.1$, \(\zeta = 1\), \(\tau = 4\), $K = 4000$; for PORTER, stepsize \(\eta_g = 0.1\), clipping parameter \(C = 1\); for PriSMA, \(\gamma = 0.025\), \(\eta = 0.025\), clipping parameters \(C_1 = 1\) and \(C_2 = 1\).
Finally, we need to select the privacy hyperparameters: for \alg we choose \(\sigma_e = 0.5\). By Theorem~\ref{thm:privacy} with privacy parameter $\delta_i = 10^{-4}$, this choice yields the privacy level \(\epsilon_i = 19.6\).
In order to match the same privacy level $\epsilon_i$, we then select for PORTER the noise standard deviation \(\sigma = 0.103\), and for PriSMA \(\sigma_0 = 0.1794\), \(\sigma_1 = 0.0155\).

We start by comparing the total computational costs. To do so, we assign a time cost $\tgrad = 0.1$ for a local gradient evaluation ($\nabla f_{i,h}$), and $\tcomm =1$ for a round of (expensive) communications. The total time costs per iteration are reported in Table~\ref{tab:time-comparison}, with \alg being the more efficient one.
In Fig.~\ref{fig:combined_results}, we present the convergence of the optimal error and classification accuracy. Notice that the x-axis is scaled according to the time complexity of each algorithm given in Table~\ref{tab:time-comparison}.
It can be seen that, under the same privacy level, \alg outperforms PORTER and PriSMA, with faster convergence and higher classification accuracy with the same privacy budgets.
% while performing only slightly worse than DP-SCAFFOLD. We remark though that DP-SCAFFOLD employs a federated architecture, which allows to propagate information from one agent to all other in one round of communication, while \alg employs a peer-to-peer architecture with slower information propagation speed.
%
\begin{table}[!ht]
\centering
\caption{Computation time over $\tau$ iterations.}
\label{tab:time-comparison}
\begin{tabular}{cc}
    \hline
    Algorithm [Ref.] & Time \\
    \hline
     PORTER \cite{li2025convergence} 
      &$\tau( \tgrad + 2 \tcomm)$ \\
        PriSMA \cite{huang2025differential}
      & $ \tau(2  \tgrad +\tcomm )$  \\
    \hline
    \alg &  $\tau \tgrad +  \tcomm$  \\
    \hline
\end{tabular}
\end{table}
\begin{figure}[!ht]
    \centering
    \begin{subfigure}{0.9\linewidth}
        \centering
        \includegraphics[width=0.8\linewidth]{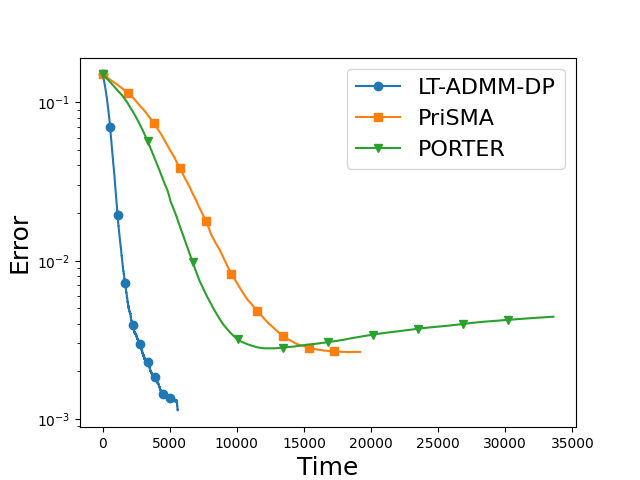}
        \caption{Errors $\|\nabla F(\bar x_k)\|$}
        \label{fig:error}
    \end{subfigure}
    \hfill
    \begin{subfigure}{0.9\linewidth}
        \centering
        \includegraphics[width=0.8\linewidth]{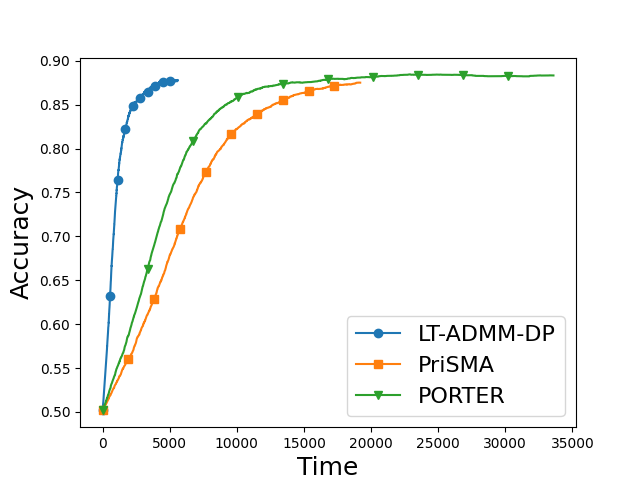}
        \caption{Classification accuracies}
        \label{fig:accuracy}
    \end{subfigure}
    \caption{Numerical results for three algorithms with differential privacy.}
    \label{fig:combined_results}
\end{figure}
%

% A possible explanation lies in the different communication mechanisms. In the federated setting, each agent transmits its noisy update only to the server, and the perturbation is aggregated centrally. By contrast, in the distributed setting, each agent exchanges noisy information with its neighbors, so the effect of the injected noise may accumulate and propagate across the network during the consensus process. From this perspective, the slight performance gap between \alg and DP-SCAFFOLD is expected, while the superior performance of \alg over PriSMA demonstrates its effectiveness in balancing communication efficiency and privacy protection, and optimization performance.

\section{Concluding Remarks}
In this paper, we presented a communication-efficient distributed algorithm with DP for nonconvex learning problems. The proposed \alg integrates local training with noisy clipped gradients, thereby improving communication efficiency while providing rigorous privacy protection for the participating agents. We established the convergence of \alg to a bounded distance from a stationary point, and further characterized its DP guarantee. Future work includes investigating adaptive clipping strategies and addressing data heterogeneity across agents.

\appendix
\subsection{Preliminary transformation}
Denote  $ \mathrm{\Phi}_k^t = \operatorname{col}\{\phi^t_{1,k}, \phi_{2,k}^t, ..., \phi_{N,k}^t\}
$, $ G(\mathrm{\Phi}_k^t) = \operatorname{col}\{ g_1(\phi_{1,k}^t), g_2(\phi_{2,k}^t),..., g_N(\phi_{N,k}^t)\}
$, $ \mathrm{F}(\mathbf{X}) = \operatorname{col}\{  f_1(x_1),   f_2(x_2),...,  f_N(x_N)\}$,  $F({x}_k) = \frac{1}{N} \sum_{i=1}^N f_i(x_k)$, $\mathbf{Z} = \operatorname{col}\{z_{ij}\}_{i,j \in \mathcal{E}} 
$. 
Define  $
\mathbf{A}= \operatorname{blk\,diag}\{ \mathbf{1}_{d_i} \}_{i \in \mathcal{V}} \in \mathbb{R}^{M \times N},
$
where $d_i = |\mathcal{N}_i|$ is the degree of node $i$, and $M = \sum_i |\mathcal{N}_i|$.
$\mathbf{P} \in \mathbb{R}^{M \times M}$ is a permutation matrix that swaps $e_{ij}$  with $e_{ji}$. 
%If there is an edge between nodes $i$, $j$, then $A^T[i,:]PA[:,j] = 1$, otherwise $A^T[i,:]PA[:,j] = 0$.
$\mathbf{A}^T\mathbf{P}\mathbf{A} = \Tilde{\mathbf{A}}$ is the adjacency matrix, $  \mathbf{A}^T\mathbf{A} = \operatorname{diag}\{ d_i \}_{i \in \mathcal{V}}$ is the degree matrix, denote $d_u$ as the largest degree among the agents. Denote the largest and smallest nonzero eigenvalue of $\mathbf{L} = \mathbf{D} - \Tilde{\mathbf{A}}$ as  $\lambda_u$ and $\lambda_l$, respectively.
In the following analysis, without loss of generality, we consider $n=1$.
The compact form of \alg is:
\begin{subequations}\label{eq:compact-admm}
\begin{equation}
\begin{aligned}
&\mathbf{X}_{k+1} = \mathbf{X}_{k} -\sum_{t=0}^{\tau -1}\left( \gamma  G(\mathbf{\Phi}_k^t) + {\beta}(\rho \mathbf{A}^T\mathbf{A} \mathbf{X}_k - \mathbf{A}^T \mathbf{Z}_k ) \right) \label{eq:compact-admm-x}
\end{aligned}
\end{equation}
\begin{equation}
    \mathbf{Z}_{k+1} =  \frac{1}{2}\mathbf{Z}_{k} - \frac{1}{2} \mathbf{P} \mathbf{Z}_k+ \rho \mathbf{P}\mathbf{A} \mathbf{X}_{k+1}. \label{eq:compact-admm-z}
\end{equation}
\end{subequations}
Moreover, we introduce the following variables $\mathbf{Y}_k= \mathbf{A}^T  \mathbf{Z}_{k} - \frac{\gamma}{\beta} \nabla F(\mathbf{\bar{X}}_k) -\rho \mathbf{D}  \mathbf{X}_k  $ and  $ \Tilde{\mathbf{Y}}_k= \mathbf{A}^T \mathbf{P}\mathbf{Z}_{k} +  \frac{\gamma}{\beta} \nabla \mathrm{F}(\bar{\mathbf{X}}_k) - \rho \mathbf{D} \mathbf{X}_k$,
where $\bar{\mathbf{X}}_k = \mathbf{1}_N \bar{x}_k$, with $\bar{x}_k = \frac{1}{N} \mathbf{1}^T \mathbf{X}_k$, and $\mathbf{D} = \mathbf{A}^T\mathbf{A} = \operatorname{diag}\{ d_i \}_{i \in \mathcal{V}}$ is the degree matrix.
Multiplying both sides of \eqref{eq:compact-admm-z} by $\mathbf{1}^T$, and using the initial condition, we obtain $\mathbf{1}^T\mathbf{A}^T\mathbf{Z}_{k+1} = \rho \mathbf{1}^T \mathbf{D} \mathbf{X}_{k+1}$ for all $k \in \N$. Then we have
\begin{equation} \label{bar_x}
\begin{aligned}
\bar{x}_{k+1} - x^* &= \bar{x}_{k}  - x^*   -\frac{\gamma}{N} \sum_{t=0}^{\tau -1} \sum_{i=1}^{N} g_i(\phi_{i,k}^{t}).
\end{aligned}
\end{equation}
And $\bar{\mathbf{Y}}_k =  \frac{\gamma}{\beta} \frac{1}{N} \mathbf{1} \mathbf{1}^T\nabla \mathrm{F}(\bar{\mathbf{X}}_k) =  \frac{\gamma}{\beta} \frac{1}{N}\mathbf{1}  \sum_{i}\nabla f_i(\bar{x}_k) = \frac{\gamma}{\beta} \mathbf{1} F(\bar{x}_k)$, and \eqref{eq:compact-admm} can be further rewritten as
\begin{equation}\label{eq:compact-admm-2}
\begin{aligned}
   & \begin{bmatrix}
     \mathbf{X}_{k+1}\\
    \mathbf{Y}_{k+1}\\
   \Tilde{\mathbf{Y}}_{k+1}
\end{bmatrix} = \begin{bmatrix}
    \mathbf{I}  &  \beta \tau \mathbf{I} & \mathbf{0}  \\
     \rho \Tilde{\mathbf{L}}  &  \rho \Tilde{\mathbf{L}}\beta \tau +  \frac{1}{2} \mathbf{I}  & - \frac{1}{2}\mathbf{I}  \\
   \mathbf{0}  &   - \frac{1}{2}\mathbf{I} &  \frac{1}{2} \mathbf{I}
\end{bmatrix} \begin{bmatrix}
     \mathbf{X}_{k}\\
    \mathbf{Y}_{k}\\
   \Tilde{\mathbf{Y}}_{k}
\end{bmatrix} 
-   \mathbf{h}_k, 
\end{aligned}
\end{equation}
where $\Tilde{\mathbf{L}} = \tilde{\mathbf{A}}- \mathbf{D},$ $ \mathbf{h}_k = 
\large[ \gamma  \sum_{t=0}^{\tau -1}( \nabla G(\mathbf{\Phi}_k^t) -  \nabla \mathrm{F}(\bar{\mathbf{X}}_k) ) ; 
\gamma  \rho  \Tilde{\mathbf{L}}   \sum_{t=0}^{\tau -1}( \nabla G(\mathbf{\Phi}_k^t) -  \nabla \mathrm{F}(\bar{\mathbf{X}}_k)   ) 
+ \frac{\gamma}{\beta} ( \nabla F(\bar{\mathbf{X}}_{k+1}) -  \nabla F(\bar{\mathbf{X}}_{k}) ) ; \frac{\gamma}{\beta} ( -\nabla F(\bar{\mathbf{X}}_{k+1}) +  \nabla F(\bar{\mathbf{X}}_{k}) ) \large] $.
% \begin{aligned}
% \mathbf{h}_k &= 
% \large[ \gamma  \sum_{t=0}^{\tau -1}( \nabla G(\mathbf{\Phi}_k^t) -  \nabla \mathrm{F}(\bar{\mathbf{X}}_k) ) ; \\&
% \gamma  \rho  \Tilde{\mathbf{L}}   \sum_{t=0}^{\tau -1}( \nabla G(\mathbf{\Phi}_k^t) -  \nabla \mathrm{F}(\bar{\mathbf{X}}_k)   ) 
% + \frac{\gamma}{\beta} ( \nabla F(\bar{\mathbf{X}}_{k+1}) -  \nabla F(\bar{\mathbf{X}}_{k}) ) ;\\& \frac{\gamma}{\beta} ( -\nabla F(\bar{\mathbf{X}}_{k+1}) +  \nabla F(\bar{\mathbf{X}}_{k}) ) \large].
% \end{aligned}$$
% Note that~\eqref{eq:compact-admm-2} can be interpreted as a linear dynamical system, with the non-linearity as input in $\mathbf{h}_k$.

\subsection{Key bounds}
\begin{lemma} \label{lem:devitaion_aver}
 Let Assumption~\ref{as:graph} hold, when  $\beta <  \frac{2}{\tau\lambda_u\rho}$,   
 \begin{equation} \label{X_Y_d} 
\|  \bar{\mathbf{X}}_k- \mathbf{X}_k \|^2 \leq \frac{18\beta \tau}{\lambda_l\rho} \|  \widehat{\mathbf{d}}_k\|^2, \quad \|  \bar{\mathbf{Y}}_k- \mathbf{Y}_k \|^2 \leq 9 \|  \widehat{\mathbf{d}}_k \|^2,
\end{equation}
and 
\begin{equation} \label{d_k_0}
\|\widehat{\mathbf{d}}_{k+1}\|^2  \leq \kappa \|\widehat{\mathbf{d}}_{k}\|^2 +
\frac{1}{1-\kappa} \|\mathbf{\widehat{h}}_{k}\|^2,
\end{equation}
where $\kappa = 1 - {\lambda_l\rho \tau \beta}/{2} <1$,
{$
\widehat{\mathbf{d}}_k = \widehat{\mathbf{V}}^{-1}
\begin{bmatrix}
\widehat{\mathbf{Q}}^T \mathbf{X}_{k};
   \widehat{\mathbf{Q}}^T \mathbf{Y}_{k};
   \widehat{\mathbf{Q}}^T\Tilde{\mathbf{Y}}_{k}
\end{bmatrix}
$, $ \widehat{\mathbf{h}}_{k} =  \widehat{\mathbf{V}}^{-1} \operatorname{blkdiag}\{\widehat{\mathbf{Q}}^T, \widehat{\mathbf{Q}}^T,\widehat{\mathbf{Q}}^T \} \mathbf{h}_k $,
where $\widehat{\mathbf{Q}} \in \mathbf{R}^{N \times (N-1)}$ satisfying $\widehat{\mathbf{Q}} \widehat{\mathbf{Q}} ^T=\mathbf{I}_N-\frac{1}{N} \mathbf{1 1}{ }^T$, $\widehat{\mathbf{Q}} ^T \widehat{\mathbf{Q}} =\mathbf{I}_{N-1}$ and $\mathbf{1}^T \widehat{\mathbf{Q}} =0$, $\widehat{\mathbf{Q}} ^T \mathbf{1}=0$. $\widehat{\mathbf{V}} \in \mathbf{R}^{3(N-1) \times 3(N-1)}$ is an invertible  matrix. 
 $\lambda_u$ and $\lambda_l$ denote the largest and the smallest nonzero eigenvalues of the Laplacian matrix $\mathcal{L}(\mathcal{G})$, respectively.}
\end{lemma}
\begin{proof}
Detailed proof can be found in \cite[Lemma 1]{ren2025communication_learning}.
\end{proof}

\begin{lemma} \label{lem:phi_k}
Let Assumptions \ref{as:local-costs}, \ref{as:graph}, and \ref{ass:Bounded gradient variation} hold, when $\beta <  \frac{2}{\tau\lambda_u\rho}$, denote
$
\|\widehat{\mathbf{\Phi}}_k\|^2=\sum_{i=1}^N \sum_{t=0}^{\tau-1}\|\phi_{i, k}^t-\bar{x}_k\|^2 =\sum_{t=0}^{\tau-1}\left\|\mathrm{\Phi}_{k}^t-\bar{\mathbf{X}}_k\right\|^2
$ we have
\begin{align} \label{phi_sgd}
& \mathbb{E} [ \|\widehat{\mathbf{\Phi}}_k\|^2 ] \leq 
( \frac{54\beta \tau^2}{\lambda_l\rho}  + 108\tau^3\beta^2  ) \mathbb{E} [  \|\widehat{\mathbf{d}}_{k}\|^2]  
\nonumber \\& +(12\tau^3 N \gamma^2 +24 \tau^2 N \gamma^2   )  \mathbb{E} [  \| \nabla F(\bar{x}_k) \|^2 ]
\nonumber \\& +  24 \tau^2 N \gamma^2 ( \zeta^2 + \sigma_e^2+  \sigma^2_f) 
\nonumber \\& = s_0 \tau^2 \|\widehat{\mathbf{d}}_{k}\|^2 + \gamma^2 s_1 \tau^2 \| \nabla F(\bar{x}_k) \|^2  + \gamma^2 s_2 \tau^2.
\end{align}
where $  s_0 = \frac{54\beta}{\lambda_l\rho}  + 108\tau\beta^2 ,  s_1 = 12\tau N  +24 N$, $  s_2 = 24 N ( \zeta^2 + \sigma_e^2+  \sigma^2_f)$,
% \begin{equation*}
%   s_0 \coloneqq \frac{54\beta \tau^2}{\lambda_l\rho}  + 108\tau^3\beta^2 ,  s1 = 12\tau^3 N  +24 \tau^2 N   
% \end{equation*}
% \begin{equation*}
%     s2 \coloneqq 24 \tau^2 N \gamma^2 ( \zeta^2 + \sigma_e^2+  \sigma^2_f)
% \end{equation*}
\end{lemma}
\begin{proof}
Using \eqref{eq:clip-gradient} we have $\| g_i(\phi_{i,k}^t) \|^2 \leq  2\zeta^2 + 2\sigma_e^2 $, together with Assumption~\ref{ass:Bounded gradient variation} we derive that
\begin{equation}
\begin{aligned}
&\mathbb{E} \| \sum_{t=0}^{\tau -1}( G(\mathrm{\Phi}_k^t) -  \nabla \mathrm{F}(\bar{\mathbf{X}}_k)   ) \|^2
\leq 4\tau^2 N ( \zeta^2 + \sigma_e^2 )
 \\&+   2\tau^2 \mathbb{E} \| \nabla \mathrm{F}(\bar{\mathbf{X}}_k)   - \mathbf{1}  \nabla F(\bar{x}_k) + \mathbf{1} \nabla F(\bar{x}_k) \|^2
 \\& \leq  4\tau^2 N ( \zeta^2 + \sigma_e^2 )
 +   4\tau^2 N \sigma^2_f + 4\tau^2 N\mathbb{E} [  \| \nabla F(\bar{x}_k) \|^2 ].
\end{aligned}
\end{equation}
Denote $\overline{G}(\mathrm{\Phi}_k^t)=\frac{1}{N}  \sum_{i=1}^{N} g_i (\phi_{i, k}^t)$ and $\overline{\nabla \mathrm{F}}(\mathrm{\Phi}_k^t)=\frac{1}{N}  \sum_{i=1}^{N} \nabla f_i (\phi_{i, k}^t)$, according to \eqref{eq:clip-gradient}, 
% $\|\sum_{t=0}^{\tau -1} \overline{G}(\mathrm{\Phi}_k^t) \|^2 \leq 2 \tau^2( \zeta^2 + \sigma^2_e). $
\begin{equation} \label{G(K)}
\begin{aligned}
&\|\sum_{t=0}^{\tau -1} \overline{G}(\mathrm{\Phi}_k^t) \|^2 \leq 2 \tau^2( \zeta^2 + \sigma^2_e).
\end{aligned}
\end{equation}
From \eqref{bar_x} we also have
$\|\nabla F(\bar{\mathbf{X}}_{k+1}) - \nabla F(\bar{\mathbf{X}}_{k}) \|^2 =N L^2 \left\| \bar{x}_{k+1} - \bar{x}_{k} \right\|^2 
 = N  L^2 \gamma^2 \|\sum_t \overline{G}(\mathrm{\Phi}_k^t)\|^2
$, it further holds that:
\begin{equation} \label{eq:h_k_0}
\begin{aligned}
& \mathbb{E} \|{\mathbf{h}}_{k}\|^2
 \leq \gamma^2 ( 1+ 2\rho^2 \|\Tilde{\mathbf{L}} \|^2)  \\& \left( 4\tau^2 N ( \zeta^2 + \sigma_e^2 )
 +   4\tau^2 N \sigma^2_f + 4\tau^2 N\mathbb{E} [  \| \nabla F(\bar{x}_k) \|^2] \right) 
 \\& + 6N L^2 \frac{\gamma^4}{\beta^2} \tau^2 (  \zeta^2 + \sigma_e )
 \\& = 4\tau^2 N\gamma^2 ( 1+ 2\rho^2 \|\Tilde{\mathbf{L}} \|^2)\mathbb{E} [  \| \nabla F(\bar{x}_k) \|^2]
 \\& +  \gamma^2 ( 1+ 2\rho^2 \|\Tilde{\mathbf{L}} \|^2)  (4\tau^2 N ( \zeta^2 + \sigma_e^2 )
 +   4\tau^2 N \sigma^2_f ) 
 \\& + 6N L^2 \frac{\gamma^4}{\beta^2} \tau^2 (  \zeta^2 + \sigma^2_e ).
\end{aligned}
\end{equation}

From \eqref{eq:compact-admm}, we obtain
\begin{equation}  \label{Phi_k}
      \mathrm{\Phi}_k^{t+1} =\mathrm{\Phi}_k^{t} + \beta \mathbf{Y}_{k}  - \gamma( G( \mathrm{\Phi}_k^t) -  \nabla \mathrm{F}(\bar{\mathbf{X}}_k)   ).
\end{equation}
% Recall that by \eqref{eq:clip-gradient},  $\| G(\Phi_k^t) - \nabla F(\Phi_k^t)  \|^2 \leq N\sigma^2$. 
Now, suppose that $\tau \geq 2$, using Jensen's inequality we obtain 
\begin{align}
&  \mathbb{E} [ \left\|\mathrm{\Phi}_k^{t+1}-\bar{\mathbf{X}}_k\right\|^2 ] \nonumber 
\\& =  \mathbb{E} [ \|\mathrm{\Phi}_k^{t}-\bar{\mathbf{X}}_k  + \beta \mathbf{Y}_{k} -\gamma( G(\mathrm{\Phi}_k^t) -  \nabla \mathrm{F}(\bar{\mathbf{X}}_k) \|^2 ] \nonumber 
\\& \leq (1+\frac{1}{\tau-1} )  \mathbb{E} [ \left\|\mathbf{\Phi}_k^{t}-\bar{\mathbf{X}}_k\right\|^2 ]  \nonumber
\\&+ \tau   \mathbb{E} [\|  \beta \mathbf{Y}_{k} - \gamma( G( \mathrm{\Phi}_k^t) -  \nabla \mathrm{F}(\bar{\mathbf{X}}_k) ) \|^2 ] \nonumber 
\\& \leq ( 1 + \frac{1}{\tau-1} )  \mathbb{E} [ \left\|\mathbf{\Phi}_{k}^t-\bar{\mathbf{X}}_k\right\|^2 ]+2 \tau \beta^2   \mathbb{E} [ \left\|  \mathbf{Y}_{k}  \right\|^2 ]\nonumber
\\& + 2 \gamma^2  (  4 N ( \zeta^2 + \sigma_e^2 )
 +   4 N \sigma^2_f + 4 N\mathbb{E} [  \| \nabla F(\bar{x}_k) \|^2 ] ). \label{phi_k_t_vr}
\end{align}

Iterating the above inequality for $t=0,..., \tau-1$, 
\begin{align*}
\mathbb{E} & [  \left\|\Phi_k^{t+1}-\bar{\mathbf{X}}_k\right\|^2] \leq (1+\frac{1}{\tau-1}) ^t \mathbb{E} [\|\mathbf{X}_{k}-\bar{\mathbf{X}}_k \|^2] +
\\& +2 \tau \beta^2 \sum_{l=0}^t  (1+\frac{1}{\tau-1}) ^l \mathbb{E} [ \| \mathbf{Y}_{k} -\bar{\mathbf{Y}}_k + \bar{\mathbf{Y}}_k \|^2 ]  \\&
+ 8 N\gamma^2  \sum_{l=0}^t (1+\frac{1}{\tau-1})^l (\zeta^2 + \sigma_e^2+  \sigma^2_f +  \mathbb{E} [  \| \nabla F(\bar{x}_k) \|^2 ]  )
\\&\leq 3 \mathbb{E} [ \left\|\mathbf{X}_{k}-\bar{\mathbf{X}}_k\right\|^2 ] 
+ 6\tau^2\beta^2 \mathbb{E} [ \left\| \mathbf{Y}_{k} -\bar{\mathbf{Y}}_k + \bar{\mathbf{Y}}_k  \right\|^2]
\\& + 24 \tau N \gamma^2  \left( \zeta^2 + \sigma_e^2+  \sigma^2_f +  \mathbb{E} [  \| \nabla F(\bar{x}_k) \|^2 ] \right),
\end{align*}
where the last inequality holds by $(1+ \frac{a}{\tau -1})^t \leq \exp(\frac{at}{\tau -1})\leq \exp(a)$ for $t\leq \tau-1$ and $a = 1$.

Summing over $t$, it follows that
\begin{equation}\label{phi_sgd_0}
\begin{split}
 &\mathbb{E} [ \|\widehat{\mathbf{\Phi}}_k \|^2 ] \\&\leq 
3\tau \mathbb{E} [\|\mathbf{X}_{k}-\bar{\mathbf{X}}_k\|^2 ]  +   (12\tau^3 N \gamma^2 +24 \tau^2 N \gamma^2   )  \mathbb{E} [  \| \nabla F(\bar{x}_k) \|^2\\& + 12\tau^3\beta^2 \mathbb{E} [\| \mathbf{Y}_{k} -\bar{\mathbf{Y}}_k \|^2 ] + 24 \tau^2 N \gamma^2 ( \zeta^2 + \sigma_e^2+  \sigma^2_f), 
\end{split}
\end{equation}
moreover, it is easy to verify that \eqref{phi_sgd_0} also holds for $\tau =1$.
Using   \eqref{X_Y_d} concludes the proof.
\end{proof}

\begin{lemma} \label{lemma:d_k}
Let Assumptions \ref{as:local-costs}, \ref{as:graph}, and \ref{ass:Bounded gradient variation} hold. When $\beta <  \frac{2}{\tau\lambda_u\rho}$, 
%it holds that {\color{red} the constants need to be redefined}
\begin{equation} \label{d_k}
\begin{aligned}
&\mathbb{E} [ \|\widehat{\mathbf{d}}_{k+1}\|^2 ] 
\\&\leq  {\kappa}  \mathbb{E} [ \|\widehat{\mathbf{d}}_{k}\|^2 ] 
+ \frac{c_1 \gamma^2 }{1-\kappa}  \mathbb{E}[\| \nabla F(\bar{x}_k) \|^2] + \frac{c_2 \gamma^2}{1-\kappa}, 
\end{aligned} 
\end{equation}
where ${\kappa}  = 1- \frac{\lambda_l\rho \tau \beta}{2}, c_1 =  ( 1+ 2\rho^2 \|\Tilde{\mathbf{L}} \|^2) 4\tau^2 N \| \widehat{\mathbf{V}}^{-1}  \|,
c_2  =  ( 1+ 2\rho^2 \|\Tilde{\mathbf{L}} \|^2)  (4\tau^2 N ( \zeta^2 + \sigma_e^2 )
 +   4\tau^2 N \sigma^2_f )  \| \widehat{\mathbf{V}}^{-1}  \|
  + 6N L^2 \frac{\gamma^2}{\beta^2} \tau^2 (  \zeta^2 + \sigma^2_e ) \| \widehat{\mathbf{V}}^{-1}  \|$.
\end{lemma}
\begin{proof}
When   $\beta <  \frac{2}{\tau\lambda_u\rho}$, using \eqref{d_k_0} and \eqref{eq:h_k_0},   we can then derive that \eqref{d_k} holds.
\end{proof}

\subsection{Proof of Theorem~\ref{theo:nonconvex_sgd_converge}}\label{noconvex_sgd_proof}
% \begin{proof}
We start our proof by recalling that the following inequality holds for all $L$-smooth functions $f$, $\forall y, z \in \R^n$ \cite{nesterov2013introductory}:
\begin{equation} \label{nonconvex_inequality}
f(y) \leq f(z) + \langle \nabla f(z), y-z \rangle + (L/2) \Vert y-z  \Vert^2. 
\end{equation}
Based on \eqref{bar_x},
substituting $ y = \bar{x}_{k+1}$ and  $ z = \bar{x}_{k}$ into \eqref{nonconvex_inequality}, since $e_{i,k}^t\sim \mathcal{N}\!\left(0_n,\;\sigma_e^2 I_n\right)$, we obtain
\begin{align} \label{eq:function}
&  \mathbb{E}[ F \left(\bar{x}_{k+1}\right)  ]  \leq  \mathbb{E} [F\left(\bar{x}_{k}\right)] +\frac{\gamma^2 L}{2}  \mathbb{E} [ \|\frac{1}{N}  \sum_t \sum_i g_i(\phi_{i, k}^t)\|^2 ]\notag
\\&-\gamma  \mathbb{E} [ \langle\nabla F\left(\bar{x}_{k}\right), \frac{1}{N} \sum_t \sum_i \mu_{i,k}^t  {g}_{i,k}^t\rangle ].
\end{align}

Denote $\mathsf{Clip}_\zeta (h) =  \frac{\zeta h}{\zeta+\|h\|} $, $\forall h \in \mathbb{R}^n$, then
\begin{equation} \label{eq:clipF}
\begin{aligned}
  &  -\gamma  \mathbb{E} [ \langle\nabla F\left(\bar{x}_{k}\right), \frac{1}{N} \sum_t \sum_i  \mu_{i,k}^t  {g}_{i,k}^t\rangle ]
  \\ & =  -\gamma  \sum_t  \mathbb{E} [ \langle\nabla F\left(\bar{x}_{k}\right), \frac{1}{N} \sum_i  \mu_{i,k}^t  {g}_{i,k}^t\rangle ]
  \\& = -\gamma  \sum_t \mathbb{E} \left\langle \nabla F(\bar{x}_{k}), \mathsf{Clip}_\zeta \left( \nabla F\left(\bar{x}_{k}\right) \right) \right\rangle 
  \\&+ \gamma  \sum_t \mathbb{E} \langle \nabla F(\bar{x}_{k}), \mathsf{Clip}_\zeta \left( \nabla F\left(\bar{x}_{k}\right) \right) - \frac{1}{N}  \sum_i \mu_{i,k}^t  {g}_{i,k}^t \rangle
  \\&  = -\gamma \tau \frac{\zeta}{\zeta+\|\nabla F(\bar x_k)\|} \|\nabla F(\bar x_k) \|^2
  \\&+ \gamma  \sum_t  \mathbb{E} \langle \nabla F(\bar{x}_{k}), \mathsf{Clip}_\zeta \left( \nabla F\left(\bar{x}_{k}\right) \right) - \frac{1}{N} \sum_i \mu_{i,k}^t  {g}_{i,k}^t \rangle,
    \end{aligned}
\end{equation}
where
\begin{align}  \label{eq:inner_product}
&\mathbb{E} \| \mathrm{Clip}_\zeta(\nabla F(\bar x_k)) - \frac{1}{N}  \sum_i \mu_{i,k}^t  \tilde{g}_i(\phi_{i,k}^t)  \| 
% \nonumber\\ =&\; \mathbb{E} \|
% \frac{1}{N} \sum_t \sum_i 
% \frac{\zeta}{\zeta+\|g_{i,k}^t\|} g_{i,k}^t
% -\frac{\zeta}{\zeta+\|\nabla F(\bar x_k)\|} \nabla F(\bar x_k)\|
\nonumber
\\ 
&=\; \mathbb{E} \|
\frac{1}{N}\sum_i
(\frac{\zeta}{\zeta+\|g_{i,k}^t\|} g_{i,k}^t-\frac{\zeta}{\zeta+\|\nabla f_i(\phi^t_{i,k})\|} g_{i,k}^t
)\nonumber \\ &
+ \frac{1}{N} \sum_i
\big(\frac{\zeta}{\zeta+\|\nabla f_i(\phi^t_{i,k})\|} g_{i,k}^t 
 \nonumber\\& 
-\frac{\zeta}{\zeta+\|\nabla f_i(\phi^t_{i,k})\|} \nabla f_i(\phi^t_{i,k}) \big) \nonumber\\
&
+ \frac{1}{N}\sum_i
\big(\frac{\zeta}{\zeta+\|\nabla f_i(\phi^t_{i,k})\|} \nabla f_i(\phi^t_{i,k})
\nonumber\\& -\frac{\zeta}{\zeta+\|\nabla F(\bar x_k)\|} \nabla f_i(\phi^t_{i,k}) \big) \nonumber\\
&
+ \frac{1}{N}  \sum_i
\big(\frac{\zeta}{\zeta+\|\nabla F(\bar x_k)\|} \nabla f_i(\phi^t_{i,k})  \nonumber\\& -\frac{\zeta}{\zeta+\|\nabla F(\bar x_k)\|} \nabla F(\bar x_k)
\big)\|.
\end{align}
For the first term in \eqref{eq:inner_product},
\begin{align}
&\frac{1}{N} \sum_i 
\mathbb{E}\|(\frac{\zeta}{\zeta+\|g_{i,k}^t\|}
-\frac{\zeta}{\zeta+\|\nabla f_i(\phi^t_{i,k})\|}) g_{i,k}^t\|
\notag\\
&=\frac{1}{N} \sum_i \mathbb{E}\|
\frac{\zeta \big(\|g_{i,k}^t\|-\|\nabla f_i(\phi^t_{i,k})\|\big)}
{(\zeta+\|g_{i,k}^t\|)(\zeta+\|\nabla f_i(\phi^t_{i,k})\|)}
\, g_{i,k}^t\|
\notag
\\ & \le
\frac{1}{N}  \sum_i \mathbb{E}
\|\|g_{i,k}^t\|-\|\nabla f_i(\phi^t_{i,k})\| \|
 \leq  \sigma_g.
\end{align}
The second term in \eqref{eq:inner_product} is bounded by $\frac{1}{N}  \sum_i
\|
\frac{\zeta}{\zeta+\|\nabla f_i(\phi^t_{i,k})\|}
\big(g_{i,k}^t-\nabla f_i(\phi^t_{i,k})\big)\|
 \leq
\frac{1}{N} \sum_i \|  g_{i,k}^t-\nabla f_i(\phi^t_{i,k}) \| \leq  \sigma_g. $
% \begin{align}
% &\frac{1}{N} \sum_t \sum_i
% \|
% \frac{\zeta}{\zeta+\|\nabla f_i(\phi^t_{i,k})\|}
% \big(g_{i,k}^t-\nabla f_i(\phi^t_{i,k})\big)\|
%  \notag\\& \;\le\;
% \frac{1}{N} \sum_t \sum_i \|  g_{i,k}^t-\nabla f_i(\phi^t_{i,k}) \| \leq \tau \sigma_g.
% \end{align}
Denote 
 $ \omega_k = \frac{\zeta}{\zeta+\|\nabla F(\bar{x}_k)\|}$,
% \begin{equation} \label{eq:omega_k}
% \omega_k = \frac{\zeta}{\zeta+\|\nabla F(\bar{x}_k)\|},
% \end{equation}
the third term  in \eqref{eq:inner_product} follows that
\begin{align}
&\frac{1}{N} \sum_t \sum_i
\| ( \frac{\zeta}{\zeta+\|\nabla f_i(\phi^t_{i,k})\|}
- \frac{\zeta}{\zeta+\|\nabla F(\bar{x}_k)\|} )
\nabla f_i(\phi^t_{i,k})
\|\notag\\& =~
\frac{1}{N}  \sum_t \sum_i
\| \frac{\zeta\big(\|\nabla f_i(\phi^t_{i,k})\|-\|\nabla F(\bar{x}_k)\|\big)}{(\zeta+\|\nabla f_i(\phi^t_{i,k})\|)(\zeta+\|\nabla F(\bar{x}_k)\|)}
\,\nabla f_i(\phi^t_{i,k})
\|\notag\\ & \le~
\frac{1}{N}  \sum_t \sum_i\frac{\zeta\big|\|\nabla f_i(\phi^t_{i,k})\|-\|\nabla F(\bar{x}_k)\|\big|}
{\zeta+\|\nabla F(\bar{x}_k)\|}
\notag\\& \le \omega_k \frac{1}{N}  \sum_t \sum_i
\big\|\nabla f_i(\phi^t_{i,k})-\nabla f_i(\bar{x}_k)\big\|\notag
\\& + \omega_k \frac{1}{N}  \sum_t\sum_i
\big\|\nabla f_i(\bar{x}_k)-\nabla F(\bar{x}_k)\big\|
\notag\\
& \le~\omega_k \frac{L}{\sqrt{N}} \| \widehat{\mathbf{\Phi}}_k \|
+\omega_k  \tau \sigma_f,
\end{align}
and the last term in \eqref{eq:inner_product} is also bounded by $\frac{\zeta}{\zeta+\|\nabla F(\bar{x}_k)\|} \frac{1}{N} \sum_t \sum_i \| \nabla f_i(\phi^t_{i,k})-
\nabla F(\bar{x}_k)\|\leq  \omega_k \frac{L}{\sqrt{N}} \| \widehat{\mathbf{\Phi}}_k \|+\omega_k  \tau \sigma_f. $
% \begin{align}
% &\|
% \frac{\zeta}{\zeta+\|\nabla F(\bar{x}_k)\|}
% (\frac{1}{N} \sum_t \sum_i \nabla f_i(x_{i,k})-
% \nabla F(\bar{x}_k))\|
% \notag\\ \leq & \omega_k \frac{\tau  L }{\sqrt{N} }\big\|\mathrm{X}_k-\bar{\mathrm{X}}_k \big\|.
% \end{align}

It follows that
\begin{align}
& \sum_t\big\| \mathrm{Clip}_\zeta(\nabla F(\bar x_k)) - \frac{1}{N}  \sum_i \mu_{i,k}^t  {g}_{i,k}^t  \big\| \notag
\\& \leq 2 \tau \sigma_g + \omega_k \frac{L}{\sqrt{N}} \| \widehat{\mathbf{\Phi}}_k \|
+\omega_k  \tau \sigma_f +  \omega_k \frac{\tau  L }{\sqrt{N} }\big\|\mathrm{X}_k-\bar{\mathrm{X}}_k \big\| \notag 
\\& = 2 \tau \sigma_g  + 2\omega_k \tau \sigma_f + 2\omega_k \frac{L}{\sqrt{N}}  \|\widehat{\mathbf{\Phi}}_k\| .
\end{align}

% $s_0 \tau^2 \|\widehat{\mathbf{d}}_{k}\|^2 + \gamma^2 s_1 \tau^2 \| \nabla F(\bar{x}_k) \|^2  + \gamma^2 s_2 \tau^2.$
Therefore, using \eqref{phi_sgd} we obtain
    \begin{align} \label{eq:inner_norm}
  &  \| \nabla F(\overline{\boldsymbol{x}}_k)\| (2 \tau \sigma_g  +2\omega_k \tau \sigma_f + 2\omega_k \tau \frac{L}{\sqrt{N}}  \frac{\|\widehat{\mathbf{\Phi}}_k\|}{\tau})  
 \nonumber \\&
 \leq  \frac{1}{4} \omega_k \tau \| \nabla F(\overline{\boldsymbol{x}}_k)\|^2 +\frac{8L^2}{N}  \omega_k \tau \frac{\|\widehat{\mathbf{\Phi}}_k\|^2}{\tau^2} \nonumber  \\&
+ 8\omega_k \tau \sigma_f^2 + \| \nabla F(\overline{\boldsymbol{x}}_k)\| 2 \tau \sigma_g
 \nonumber \\&
 \leq (\frac{1}{4}  \omega_k \tau  + \frac{8L^2}{N}  \omega_k \gamma^2 s_1 \tau ) \| \nabla F(\overline{\boldsymbol{x}}_k)\|^2 + 
   \frac{8L^2}{N}  \omega_k s_0 \tau \|  \widehat{\mathbf{d}}_k\|^2
 \nonumber \\& + 8\omega_k \tau \sigma_f^2 + \frac{8L^2}{N}  \omega_k \gamma^2 \tau s_2+ \| \nabla F(\overline{\boldsymbol{x}}_k)\| 2 \tau \sigma_g.
\end{align}

Using \eqref{d_k} we obtain that $ \mathbb{E} [ \|\widehat{\mathbf{d}}_k\|^2 ] \leq {\kappa}^k \mathbb{E} [ \|\widehat{\mathbf{d}}_0\|^2]+ \frac{c_1 \gamma^2 }{1-\bar{\kappa}} \sum_{\ell=0}^{k-1}{\kappa}^{k-1-\ell} \| \nabla F(\bar{x}_\ell) \|^2 + \frac{c_2 \gamma^2}{1-{\kappa}} $.
% \begin{align}
% &\mathbb{E} [ \|\widehat{\mathbf{d}}_k\|^2 ] \leq {\kappa}^k \mathbb{E} [ \|\widehat{\mathbf{d}}_0\|^2]+ \frac{c_1 \gamma^2 }{1-\bar{\kappa}} \sum_{\ell=0}^{k-1}{\kappa}^{k-1-\ell} \| \nabla F(\bar{x}_\ell) \|^2 \nonumber \\&+ \frac{c_2 \gamma^2}{1-{\kappa}} 
% \end{align}
Summing this inequality over $k=0, \ldots, K-1$, it follows that 
\begin{equation} \label{sum_d}
\begin{aligned}
&\sum_{k=0}^{K-1}  \mathbb{E} [\|\widehat{\mathbf{d}}_k\|^2 ]
\leq \frac{\|\widehat{\mathbf{d}}_0\|^2}{1-{\kappa}  }+\frac{c_1 \gamma^2 }{(1-{\kappa})^2} \sum_{k=0}^{K-1} \| \nabla F(\bar{x}_k) \|^2\\& + \frac{c_2 \gamma^2 K}{1-{\kappa}}.
\end{aligned}
\end{equation}

Combining \eqref{G(K)}, \eqref{eq:function}, \eqref{eq:clipF} and 
\eqref{eq:inner_norm}
\begin{align}
&    F \left(\bar{x}_{k+1}\right)  - F\left(\bar{x}_{k}\right)  \leq  -\gamma \tau \omega_k \| \nabla F(\bar x_k) \|^2 \nonumber
  \\&+ \gamma  (\frac{1}{4}  \omega_k \tau  + \frac{8L^2}{N}  \omega_k \gamma^2 s_1 \tau ) \| \nabla F(\overline{\boldsymbol{x}}_k)\|^2  \nonumber
  \\& + \gamma (8 \omega_k \tau \sigma_f^2 + \frac{8L^2}{N}  \omega_k \gamma^2 \tau s_2+ \| \nabla F(\overline{\boldsymbol{x}}_k)\| 2 \tau \sigma_g )\nonumber
  \\& +{\gamma^2 L}\tau^2( \zeta^2 + \sigma_e)+ \gamma \frac{8L^2}{N}  \omega_k s_0 \tau \|  \widehat{\mathbf{d}}_k\|^2.
\end{align}
Denote $\tilde{F}\left(\bar{x}_{k}\right) = F\left(\bar{x}_{k}\right)-F(x^*)$. 
Summing over $k=0,1, \ldots, K-1$,  using $\tilde{F}\left(\bar{x}_{k}\right) \geq 0$,
\begin{align}
&  - \tilde{F}\left(\bar{x}_{0}\right)  \leq  -\gamma \tau \sum_{k=0}^{K-1} \omega_k \| \nabla F(\bar x_k) \|^2 \nonumber
  \\&+ \gamma \sum_{k=0}^{K-1} (\frac{1}{4} \omega_k \tau + \frac{8L^2}{N}  \omega_k \gamma^2 s_1 \tau) \| \nabla F(\overline{\boldsymbol{x}}_k)\|^2 \nonumber
  \\&+ \gamma\frac{8L^2}{N} \omega_k s_0 \tau \notag
  \\& ( \frac{\|\widehat{\mathbf{d}}_0\|^2}{1-\kappa  }+\frac{c_1 \gamma^2 \sum_{k=0}^{K-1} \| \nabla F(\bar{x}_k) \|^2}{(1-\kappa)^2}  + \frac{c_2 \gamma^2 K}{1-\kappa} )
  \nonumber
  \\& + \gamma \sum_{k=0}^{K-1}  (  8\omega_k \tau \sigma_f^2 + \frac{8L^2}{N}  \omega_k \gamma^2 \tau s_2+ \| \nabla F(\overline{\boldsymbol{x}}_k)\| 2 \tau \sigma_g )\nonumber
  \\& +{\gamma^2 LK} \tau^2( \zeta^2 + \sigma^2_e).
\end{align}

Let $\gamma$ satisfies that
\begin{equation} \label{eq:gamma1}
\frac{1}{4}   + \frac{8L^2}{N}   \gamma^2 s_1 +  \frac{8L^2}{N} s_0  \frac{c_1 \gamma^2 }{(1-\kappa)^2} \leq \frac{1}{2},
\end{equation}
since $\omega_k \leq 1$,
\begin{align} \label{eq:relation_F}
&  \tilde{F}\left(\bar{x}_{0}\right) 
\nonumber \\& \leq  -\frac{\gamma \tau}{2} \sum_{k=0}^{K-1} \omega_k \| \nabla F(\bar x_k) \|^2 + \gamma \sum_{k=0}^{K-1}  \| \nabla F(\overline{\boldsymbol{x}}_k)\| 2 \tau \sigma_g  \nonumber
\\&+ \gamma \frac{8L^2}{N}  s_0 \tau ( \frac{\|\widehat{\mathbf{d}}_0\|^2}{1-\kappa  }+ \frac{c_2 \gamma^2 K}{1-\kappa} ) 
\nonumber \\& +{\gamma^2 LK} \tau^2( \zeta^2 + \sigma^2_e) + 8 K \gamma \tau \sigma_f^2 + \frac{8L^2}{N} K\gamma^3 \tau s_2.
\end{align}

% There is a constant (the accuracy we want to obtain), 
Let  $\zeta > 8 \sigma_g $, 
if $ \| \nabla F(\bar x_k) \| \geq \zeta$, which we denote as $  k \in \mathbf{K_1}$, we have $\omega_k \big\|\nabla F(\bar x_k)\big\|^2=
\frac{\zeta \|\nabla F(\bar  x_k  ) \|^2}
{\zeta + \|\nabla F(\bar x_k )\|} \ge \frac{\zeta}{2}\|\nabla F(\bar x_k)\| $
% $ \omega_k \big\|\nabla F(\bar x_k)\big\|^2=
% \frac{\zeta \|\nabla F(\bar  x_k  ) \|^2}
% {\zeta + \|\nabla F(\bar x_k )\|} \ge \frac{\zeta}{2}\|\nabla F(\bar x_k)\|$ and $- \frac{\omega_k}{2} \| \nabla F(\bar x_k) \|^2 + \| \nabla F(\overline{\boldsymbol{x}}_k)\| 2 \tau \sigma_g \leq -(\frac{\zeta}{4} - 2\tau \sigma_g) \| \nabla F(\bar x_k) \| $.
% \begin{align*}
% \omega_k \big\|\nabla F(\bar x_k)\big\|^2=
% \frac{\zeta \|\nabla F(\bar  x_k  ) \|^2}
% {\zeta + \|\nabla F(\bar x_k )\|} \ge \frac{\zeta}{2}\|\nabla F(\bar x_k)\|
% \end{align*}
and
 \begin{align}
& - \frac{\omega_k \tau}{2} \| \nabla F(\bar x_k) \|^2 + \| \nabla F(\overline{\boldsymbol{x}}_k)\| 2 \tau \sigma_g 
\nonumber
\\& \leq -(\frac{\zeta \tau}{4} - 2\tau \sigma_g) \| \nabla F(\bar x_k) \|.
 \end{align}
Otherwise if $ \| \nabla F(\bar x_k) \| < \zeta$, which we denote as $  k \in \mathbf{K_2}$, then
$
\omega_k \big\|\nabla F(\bar x_k)\big\|^2=
\frac{\zeta \|\nabla F(\bar  x_k  ) \|^2}
{\zeta + \|\nabla F(\bar x_k )\|} \ge \frac{1}{2}\|\nabla F(\bar x_k)\|^2,
$ we obtain that 
\begin{align}
&
- \frac{\omega_k \tau}{2} \| \nabla F(\bar x_k) \|^2 + \| \nabla F(\overline{\boldsymbol{x}}_k)\| 2 \tau \sigma_g 
\nonumber
\\& \leq -\frac{\tau }{4}\|\nabla F(\bar x_k)\|^2 + 2 \zeta \tau \sigma_g. 
\end{align}
Combining the above two cases and using \eqref{eq:relation_F}, we obtain that 
   \begin{align} 
&  \sum_{\substack{0 \leq k \leq K-1 \\ k \in \mathbf{K_1}}} (\frac{\zeta}{4} - 2 \sigma_g )\|\nabla F(\bar x_k)\| +  \sum_{\substack{0 \leq k \leq K-1 \\ k \in \mathbf{K_2}}} \frac{1}{4}\|\nabla F(\bar x_k)\|^2  \nonumber
\\& \leq   \frac{F\left(\bar{x}_{0}\right)-F(x^*) }{\gamma \tau} +\sum_{\substack{0 \leq k \leq K-1 \\ k \in \mathbf{K_2}}} 2\zeta  \sigma_g 
\nonumber 
\\&+  \frac{8L^2}{N}  s_0  ( \frac{\|\widehat{\mathbf{d}}_0\|^2}{1-\kappa  }+ \frac{c_2 \gamma^2 K}{1-\kappa} ) + \frac{8L^2}{N} K\gamma^2  s_2
\nonumber \\& +{\gamma LK} \tau( \zeta^2 + \sigma^2_e) + 8 K  \sigma_f^2.
\end{align}
Therefore, \eqref{the:main_convergence} holds.
% \end{proof}

\bibliographystyle{IEEEtran}
\bibliography{reference,IDS_Publications}
\end{document}